\documentclass[conference]{IEEEtran}
\usepackage{blindtext, graphicx}
\usepackage{graphicx}
\usepackage{algorithm}
\usepackage{algorithmic}
\usepackage{amsthm}
\usepackage{amssymb}
\usepackage{multicol}
\usepackage{amsmath}
\usepackage{multicol}
\usepackage{color, soul}
\usepackage{times}

\ifCLASSINFOpdf
\else
\fi

\newtheorem{thm}{Theorem}
\newtheorem{lem}{Lemma}
\newtheorem{ex}{Fact}

\begin{document}
\title{Multi-armed Bandit Problem with Known Trend}

\author{\IEEEauthorblockN{Djallel Bouneffouf}
\IEEEauthorblockA{Canada's Michael Smith Genome Sciences Centre,\\
University of British Columbia,\\
Vancouver, British Columbia, Canada\\
Email: dbouneffouf@bcgsc.ca\\
}
\and
\IEEEauthorblockN{Raphael F\'eraud}
\IEEEauthorblockA{Orange Labs, \\
2 av. Pierre Marzin,\\
22300 Lannion (France)\\
Email: Raphael.feraud@orange.com\\
}
}


%Telephone: (800) 555--1212, Fax: (888) 555--1212}


\maketitle

\begin{abstract}
We consider a variant of the multi-armed bandit model, which we call multi-armed bandit problem with known trend, where the gambler knows the shape of the reward function of each arm but not its distribution. This new problem is motivated by different on-line problems like active learning, music and interface recommendation applications, where when an arm is sampled by the model the received reward change according to a known trend. By adapting the standard multi-armed bandit algorithm UCB1 to take advantage of this setting, we propose the new algorithm named Adjusted Upper Confidence Bound (A-UCB) that assumes a stochastic model. We provide upper bounds of the regret which compare favourably with the ones of UCB1. We also confirm that experimentally with different simulations.
\end{abstract}


% \fi
\IEEEpeerreviewmaketitle

\section{Introduction}
The basic formulation of the Multi-Armed Bandit (MAB) problem can be described as follows: there are K arms, each having a fixed, unknown and independent probability-distribution of reward. At each step, a player chooses an arm and receives a reward. This reward is drawn according to the selected arm's distribution and it is independent of previous actions. Under this assumption, many policies have been proposed to optimize the long-term accumulated reward. 

A challenging variant of the MAB problem is the non-stationary bandit problem where the player must decide which arm to play while facing the possibility of a changing environment. We study here a special case of this model where the rewards of each arm of the bandit follow a known function. 
In this setting, is it possible to adapt standard bandit algorithms to take advantage of this new setting?

The answer of this question is interesting by itself: it could open new doors from a theoretical point of view. But the real motivation is operational: knowing the shape of the reward function assumption is realistic for several real-world problems like on-line active learning, A/B testing and music recommendation.
For instance, in \cite{BouneffoufLUFA14}, the analysis of the active learning problem led the authors to model the active learning problem as a MAB problem. They cluster at first the input space: each cluster is considered as an arm. In this setting the authors find that the more an area is sampled by the model the less is the received reward.
In \cite{hu2011nextone}, the authors study the recommendation of music where they observe that the interest of a user to a music follows the inverse of an exponential function called forgetting curve, where the more a music is heard the lesser it is interesting. Another model that follows a reward with known function was studied in \cite{rosman2014user}. The authors observe that when they propose a new interface to a user, at the beginning, the user dislikes it, but after using it several times, the user begin to like it, which means that if we model this problem as a bandit where the interface is an arm, we can say that the reward of the arm start to be bad at the beginning and it increases by time.

From the three above examples, we can say that all these problems can be modeled as new bandit problem called ``Multi-armed Bandit Problem with Known Trend" where each arm follow a known trend reward function. For instance, in the first two examples the rewards follow a decreasing function and the third one follows a sigmoid function. 
In this setting we propose to study this new model derived from this problem, by adapting the existing algorithm to the new setting and analysing their regret. Finally, we evaluate the proposed algorithms through different simulations.

The remaining of the paper is organized as follows. Section \ref{sec:related} reviews related works. Section \ref{sec:statement} describes the setting MAB model with known trend reward function and the proposed algorithm A-UCB. Then we proof its regret in Section \ref{sec:Analysis}. The experimental evaluation through different simulations is illustrated in Section \ref{sec:experimental}. The last section concludes the paper and points out possible directions for future works.

\section{RELATED WORK}
\label{sec:related}

This section provides an overview on the MAB problem related to our work. In the bandit problem, each arm delivers rewards that are independently drawn from an unknown distribution. An efficient solution based on optimism in the face of uncertainty principle has been proposed proposed by Lai and Robbins \cite{LaiRobbins1985} compute an index for each arm and they choose the arm with the highest index. 

Our work is an adaptation of these classes of policies for MAB Problem with known trend reward function. Our work is most related to the study of dynamic versions of the MAB where either the set of arms or their expected reward may change over time. 
There are several applications, including active learning, music and interface recommendation, where the rewards are far from being stationary random sequences. A solution to cope with non-stationary is to drop the stochastic reward assumption and assume the reward sequences to be chosen by an adversary. Even with this adversarial formulation of the MAB problem, a randomized strategy like EXP3 provides the guarantee of a minimal regret 
\cite{4,AllesiardoFB14}.

Another work done in \cite{garivier2008upper} considers the situation where the distributions of rewards remain constant over epochs and change at unknown time instants. They analyze two algorithms: the discounted UCB and the sliding-window UCB and they establish for these two algorithms an upper-bound for the expected regret by upper-bounding the expectation of the number of times a suboptimal arm is played. They establish a lower-bound for the regret in presence of abrupt changes in the arms reward distributions.

Similar to \cite{garivier2008upper}, authors in \cite{MellorS13} propose a Thompson Sampling strategy equipped with a Bayesian change point mechanism to tackle this problem. They develop algorithms for a variety of cases with constant switching rate: when switching occurs all arms change (Global Switching), switching occurs independently for each arm (Per-Arm Switching), when the switching rate is known and when it must be inferred from data.

Motivated by task scheduling, the author in \cite{gittins1979bandit} proposed a policy where only the state of the arm currently selected can change in a given step, and proved its optimality for time discounting. This result gave rise to a rich line of work.
For example, \cite{whittle1988restless,nino2001restless} studied the restless bandits, where the states of all arms can change in each step according to an arbitrary stochastic transition function.

To deal with the partial information nature of the bandit problem, in Adapt-Eve \cite{Hartland2006}
 the mean reward of the estimated best arm is monitored.  The drawback of this approach is that it does not tackle the case of a suboptimal arm becoming the best arm.

In \cite{BurtiniLL15} author study specific classes of drifting restless bandits selected for their relevance to modelling an online website optimization process. The contribution was a feasible weighted least squares technique capable of utilizing contextual arm parameters while considering the parameter space drifting non-stationary within reasonable bounds.

Another line of work studies the non-stationary reward of arms by considering that each arm has a finite lifetime. In this mortal bandits setting, each disappearing arm changes the set of available arms. Several algorithms were proposed and analyzed in \cite{NIPSMortal1} for mortal bandits under stochastic reward assumptions. 
In sleeping bandits problem \cite{kanade2009sleeping}, the set of strategies is fixed but only a subset of them available in each step. In their model they study the mixture-of-experts paradigm, where a set of experts is specified in each time period. The goal of the algorithm is to choose one expert in each time period to minimize regret against the best mixture of experts.

Our new model can be considered an extension of the work done in \cite{NIPSMortal1,kanade2009sleeping}, the main difference is in the fact that, in our case the reward function of each arm can follow any function not specially a decreasing function. Our model can also be a specification of the general model of restless bandits with a known shape of the reward function. 

\section{PROBLEM STATEMENT}
\label{sec:statement}

In this section, we present the algorithm and the main theorem that bounds its regret. Before that, we first provide the setting of our problem.
In the MAB setting, to maximize his gain the player has to find the best arm as soon as possible, and then exploit it. In our setting, the rewards follow a known function. When the player has found the best arm, he knows that this arm  will be the best just for a certain period of time. The player needs to re-explore at each time to find the next best arm. In the following, we define our setting.

Let $r_i(1),....,r_i(n)$ be a sequence of independent draws of the random variable $r_i \in [0,1]$ with $n$ the number of trials and let $\mu_i=E[r_i]$ be its mean reward.
At each time $t$, the player chooses an arm $i \in \{1,...,K\}$ to play according to a (deterministic or random) policy $\phi$ based on sequence of plays and reward, and obtains a non-stationary reward $z(t)$ where $z(t)= r_{i_t}(t) \cdot D(n_{i_t}(t))$, where $D(n_{i_t}(t))$ is a trend reward function assumed to be known, $n_{i_t}(t)$ is the number of times $i$ is played and $r_{i_t}(t)$ is the stationary reward for arm $i$ at time $t$.

A dynamic policy can be defined as function such that $(\phi: t\mapsto n_{1}(t),...,n_{K}(t))$
 or $\phi: H_{t-1} \mapsto K$, 
 
 where $H_{t-1}$ is the history of rewards known at time $t$.
 
By applying a policy $\phi$, at time $t$ a sequence of choices  is obtained $(1,2,...,t) \in [K]^t$.

At time $t$, the gain of the policy $\phi$ is:
\begin{eqnarray*}                                     
G_{\phi} (t)=\sum_t z(t)=\sum_t r_{i_t}(t) D(n_{i_t}(t))
\end{eqnarray*}
The performance of a policy $\phi$ is measured in terms of regret in the first $T$ plays, which is defined as the expected difference between the total rewards collected by the optimal policy $\phi^*$ (playing at each time instant the arm $i^*$ with the highest expected reward) and the total rewards collected by the policy $\phi$.

The objective is to minimize the regret $R(T)$ at time $T$, where $T$ is the time horizon.

\textbf{Definition 1.} The expected regret after $T$ plays may be expressed as:
\begin{equation}
 \label{eq:prn}
E[R(T)]=E[G^*(T)]-E[G(T)]\nonumber
\end{equation} 
where $E[G^*(T)]= \sum^k_{i=1} \mu_i \sum^{n^*_i(T)}_{s=1} D(s)$, is the optimal gain expectation, and $E[G(T)]= \sum^k_{i=1} \mu_i E[\sum^{n_i(T)}_{s=1} D(s)]$, the expected gain got by the policy $\phi$. Note that we distribute the expectation because $r_{i}(t)$ and $n_i(t)$ are independent. 

\textbf{Definition 2.} The optimal policy in any time $\phi^*$ consists in always playing the arm $i^* \in\{1,...,K\}$ with largest expected reward:
\begin{equation}
 \label{eq:prn}
i^* = argmax_{i}[\mu_i \cdot D(n_{i}(t))]\nonumber
\end{equation} 
where $\mu_i$ is the expectation of the reward $r_i(t)$.

\textbf{Definition 3.} $F$ is  the cumulative function of $D(s)$ and is expressed as follows: $F(n_i(T))=\sum^{n_i(T)}_{s=1} D(s)$, notice that in the demonstration of the theorem 1, $F$ is assumed to be Lipschitz.

\subsection{Upper Confident Bound with Known Trend}
\label{subsec:UCB}

To adapt the UCB algorithm for the news setting. The proposed A-UCB algorithm computes at each trial $t$ an index $I(i)= (\hat{\mu}_i+c(i)).D(n_{i}(t))$ for each arm $i$, where $c(i)$ is the corresponding confidence interval, so that: 
\begin{eqnarray*}
c(i)=\sqrt{\frac{2 \times log(t)}{n_{i}(t)}}
\end{eqnarray*}

The UCB index is multiplied by $D(n_{i}(t))$ to stop playing the supposed optimal arm when its rewards become suboptimal.

\begin{algorithm}[H]
   \caption{The A-UCB algorithm}
%   \label{alg:2} 
\label{alg:ucb} 
\textbf{Require:} Arm $i \in I$.

\textbf{Foreach} t = 1, 2, . . . ,T \textbf{do}

 Select arm $i_t= argmax_i (\hat{\mu}_i+c(i)).D(n_{i}(t))$ 
 
 Observe reward $r_i(t)$

\textbf{End}     
\end{algorithm}
\begin{thm}\label{theorem:Theorem2}
For the bandit problem with known trend, the accumulated expected regret $R$ of A-UCB policy is bounded by 
\begin{eqnarray*}
E[R(T)]\leq max_i \mu_i D_{max} \sum_{i:n^*_i(T)>n_i(T)} \frac{8 ln T}{\Delta_i^{'2}}+ K \frac{\pi^2}{3}
\end{eqnarray*}
\end{thm}

where $\Delta_i'=  \frac{D_{min}}{D_{max}} \mu_{i_t^*} - \mu_{i_t}$ with $D_{min}$ and $D_{max}$ two lipschitz constants.

\subsection{Proof of Theorem 1}
\label{sec:Analysis}

To prove the theorem 1, we start by upper bounding $E[R(T)]$ by $max_i \mu_i D_{max}$

$\sum_{i:n_i(T)>n^*_i(T)} |E[n_i(T)]-n^*_i(T)|$ in lemma \ref{sec:Lemma2} and after that we bound 

$\sum_{i:n_i(T)>n^*_i(T)} |E[n_i(T)]-n^*_i(T)|$ in lemma \ref{sec:Lemma3}. 

We start by introducing some fact.

\begin{ex} (Chernoff-Hoeffding bound) Let $X_i\in [0,1]$ an independent random variables with $\mu=E[X_i]$. 
\begin{eqnarray*}
Pr(\frac{1}{n_{i}(t)} \sum^{n_{i}(t)}_{i=1} X_i-\mu \geq \epsilon)\leq e^{-2n_{i}(t)\epsilon^2}
\end{eqnarray*}
, and 
\begin{eqnarray*}
Pr(\frac{1}{n_{i}(t)} \sum^{n_{i}(t)}_{i=1} X_i-\mu \leq -\epsilon)\leq e^{-2n_{i}(t)\epsilon^2}
\end{eqnarray*}
\end{ex} 

\begin{ex} (L-lipschitz) Let $M$ a part from $R$, $f:M\longrightarrow R$ a function and $L$ a real positive.
\begin{eqnarray*}
\forall(x,y) \in M^2, |f(x)-f(y)|\leq L |x-y|
\end{eqnarray*}
\end{ex} 
 
%\begin{ex} (L-lipschitz) Let $M$ a part from $R$, $f:M\longrightarrow R$ a function and $L$ a real positive. 
%$\forall(x,y) \in M^2, |f(x)-f(y)|\leq L |x-y|$
%\end{ex} 

%\subsubsection{Lemma 1.}
\begin{lem} 
\label{sec:Lemma2}
\begin{eqnarray*}
E[R(T)]\leq max_i \mu_i D_{max} \sum_{i:n_i(T)>n^*_i(T)} |E[n_i(T)]-n^*_i(T)|
\end{eqnarray*}
\end{lem} 

\begin{proof}

We know that $E[R(T)]=E[G^*(T)]-E[G(T)]$

\begin{eqnarray*}\Rightarrow E[R(T)]=\sum^k_{i=1} \mu_i \sum^{n^*_i(T)}_{s=1} D(s)-\sum^k_{i=1} \mu_i E[\sum^{n_i(T)}_{s=1} D(s)]
\end{eqnarray*}

\begin{eqnarray*}\Rightarrow E[R(T)]= \sum^k_{i=1} \mu_i [F(n^*_i(T))- E[F(n_i(T))]]
\end{eqnarray*}

$F$ is monotonous. 

\begin{eqnarray*}
\Rightarrow E[R(T)] \leq \sum^k_{i=1} \mu_i [ F(n^*_i(T))- F(E[n_i(T)])]
\end{eqnarray*}

\begin{eqnarray*}
\begin{split}
\Rightarrow E[R(T)]&= \sum_{i:n^*_i(T)>n_i(T)} \mu_i [F(n^*_i(T))-F(E[n_i(T)])]\\
&+\sum_{i:n^*_i(T)\leq n_i(T)} \mu_i [F(n^*_i(T))-F(E[n_i(T)])]
\end{split}
\end{eqnarray*}

$F$ is an increasing function.

\begin{eqnarray*}
\begin{split}
\Rightarrow E[R(T)] &\leq  \sum_{i:n^*_i(T)>n_i(T)} \mu_i |F(n^*_i(T))-F(E[n_i(T)])|\\
&+\sum_{i:n^*_i(T)\leq n_i(T)} \mu_i |F(n^*_i(T))-F(E[n_i(T)])|
\end{split}
\end{eqnarray*}

\begin{eqnarray*}
\sum_{i:n^*_i(T)>n_i(T)} \mu_i |F(n^*_i(T))-F(E[n_i(T)])|=0,
\end{eqnarray*}
because there is no regret when $n_i(T) \leq n^*_i(T)$ 

\begin{eqnarray*}
E[R(T)] \leq  \sum_{i:n^*_i(T)\leq n_i(T)} \mu_i |F(n^*_i(T))-F(E[n_i(T)])|
\end{eqnarray*}

When $F$ is $D_{max}-Lipschitz$ we have,

\begin{eqnarray*}
E[R(T)]\leq \sum_{i:n_i(T)>n^*_i(T)} \mu_i \cdot D_{max} \cdot |n^*_i(T)-E[n_i(T)]|
\end{eqnarray*}
where $D_{max}=max_sD(s)$ and $s\in [1,T]$

\begin{eqnarray*}
\begin{split}
\Rightarrow E[R(T)] &\leq max_i \mu_i D_{max} \\
 & . \sum_{i:n_i(T)>n^*_i(T)} |n^*_i(T)-E[n_i(T)]|
\end{split}
\end{eqnarray*}

\begin{eqnarray*}
\begin{split}
\Rightarrow E[R(T)] & \leq max_i \mu_i D_{max} \\
 & . \sum_{i:n_i(T)>n^*_i(T)} |E[n_i(T)]-n^*_i(T)|
\end{split}
\end{eqnarray*}
\end{proof}

In lemma \ref{sec:Lemma3} we are going to bound the $E[n_i(T)]-n^*_i(T)$ which is equal to 

$\sum^T_{t=1} Pr(i(t)=i) 1\{n_i(t)\geq n^*_i(t)\}$

%\subsubsection{Lemma 2.}
\begin{lem} 
\label{sec:Lemma3}
\begin{eqnarray*}
\sum^T_{t=1} Pr(i(t)=i) 1\{n_i(t)\geq n^*_i(t)\}\leq \frac{8log T}{\Delta^{'2}_i}+\frac{\pi^2}{3}
\end{eqnarray*}
\end{lem}

\begin{proof}

Let $B(i)=\sqrt{\frac{2 log t}{n_{i_t}(t)}}$  and from hoeffding we know that with $1\leq n_{i_t}(t) \leq t$, we have
\begin{equation}
 \label{eq:pr21}
Pr(\hat{\mu}_{i_t}+B(i)\leq \mu_{i_t})\leq e^{-4log(t)}=t^{-4}
\end{equation} 
and 
\begin{equation}
 \label{eq:pr3}
Pr(\hat{\mu}_{i_t}-B(i) \ge \mu_{i_t})\leq e^{-4log(t)}=t^{-4}
\end{equation} 
Suppose that at the time $t$ the empiric mean of each arm is in there confidence interval,
\begin{equation}
 \label{eq:pr4}
D(n_{i_t}(t))\cdot(\mu_{i_t}-B(i)) \overset{(a)}{\leq} D(n_{i_t}(t))\cdot \hat{\mu}_{i_t}
\end{equation} 

\begin{equation}
 \label{eq:pr5}
\overset{(b)}{\leq} D(n_{i_t}(t))\cdot (\mu_{i_t}+B(i))\nonumber
\end{equation} 

Let $i_t$ a suboptimal arm and $i_t^*$ an optimal arm. If the arm $i_t$ is played at the time $t$, its means that

\begin{eqnarray*}
\begin{split}
D(n_{i_t}(t))\cdot(\hat{\mu}_{i_t}+\sqrt{\frac{2 log t}{n_{i_t}(t)}}) \geq & D(n_{i_t^*}(t))\cdot(\hat{\mu}_{i_t^*} +\sqrt{\frac{2 log t}{n_{i_t^*}(t)}}) \geq \\
& D(n_{i_t^*}(t))\cdot \mu_{i_t^*} 
\end{split}
\end{eqnarray*} 

From (\ref{eq:pr4}), 
\begin{eqnarray*}
\begin{split}
D(n_{i_t}(t))\cdot(\hat{\mu}_{i_t}+\sqrt{\frac{2 log t}{n_{i_t}(t)}})+\\  D(n_{i_t}(t))\cdot\sqrt{\frac{2 log t}{n_{i_t}(t)}} \geq D(n_{i_t^*}(t))\cdot \mu_{i_t^*} 
\end{split}
\end{eqnarray*}

\begin{eqnarray*}
\Rightarrow D(n_{i_t}(t))\cdot(\mu_{i_t}+2\sqrt{\frac{2 log t}{n_{i_t}(t)}}) \geq D(n_{i_t^*}(t))\cdot \mu_{i_t^*} 
\end{eqnarray*}

\begin{eqnarray*}
\Rightarrow 2\sqrt{\frac{2 log t}{n_{i_t}(t)}} D(n_{i_t}(t))\geq D(n_{i_t^*}(t))\cdot\mu_{i_t^*} - D(n_{i_t}(t))\cdot \mu_{i_t}
\end{eqnarray*}

\begin{eqnarray*}
\Rightarrow 2\sqrt{\frac{2 log t}{n_{i_t}(t)}} \geq  \frac{D(n_{i_t^*}(t))}{D(n_{i_t}(t))} \cdot\mu_{i_t^*} - \mu_{i_t} 
\end{eqnarray*}

with $ D_{min} \leq D(n_{i_t}(t))\leq D_{max} $
and
\begin{eqnarray*}
\Delta_i'=  \frac{D_{min}}{D_{max}} \mu_{i_t^*} - \mu_{i_t}
\end{eqnarray*}

we have, 
\begin{eqnarray*}
2\sqrt{\frac{2 log t}{n_{i_t}(t)}} \geq \Delta_i
\end{eqnarray*}
, which means that 
\begin{eqnarray*}
n_{i_t}(t) \leq  \frac{8log t}{\Delta_i^{'2}}
\end{eqnarray*}

For all integer $u$, we have:
\begin{eqnarray}
\label{eq:pr7}
\begin{split}
n_{i_t}(t) &\leq u+ \sum^T_{t=u+1} 1\{I_t=i_t;n_{i_t}(t)>u\} \leq u \\
& + \sum^T_{t=u+1} 1\{\exists n_{i_t}(t) : u < n_{i_t}(t) \leq t,  \\
&\exists n_{i_t^*}(t) : 1 \leq n_{i_t^*}(t) \leq t, \\ 
& B_{t,n_{i_t}(t)}(i_t)\geq B_{t,n_{i_t^*}(t)}(i_t^*) \}
\end{split}
\end{eqnarray}

Now, from (\ref{eq:pr7}) we can say that (\ref{eq:pr4}) $\Rightarrow n_{i_t}(t)<\frac{8logt}{\Delta^2i_t}$ or the inequality (a) or (b) in (\ref{eq:pr4}) is not satisfied. 
Then by choosing $u=\frac{8logt}{\Delta^2_i}$, we infer that (a) or (b) is not satisfied. But from (\ref{eq:pr21}), the inequality (a) is not satisfied with a probability $\leq t^{-4}$, end from (\ref{eq:pr3}) the inequality (b) is not satisfied with a probability $\leq t^{-4}$.
By taking the arithmetic mean in both sides of (4), 
 
\begin{eqnarray*}
E[n_{i_t}(t)]\leq \frac{8log t}{\Delta_i^{'2}} + \sum^{T}_{t=u+1}\sum^{t}_{s=u+1} t^{-4} \sum^{t}_{s=1} t^{-4}
\end{eqnarray*}

\begin{eqnarray*}
\Rightarrow E[n_{i_t}(t)]\leq \frac{8log t}{\Delta_i^{'2}}+\frac{\pi^2}{3}
\end{eqnarray*}

\begin{eqnarray*}
\Rightarrow \sum^T_{t=1} Pr(i(t)=i) 1\{n_i(t)\geq n^*_i(t)\} \leq \frac{8log t}{\Delta_i^{'2}}+\frac{\pi^2}{3}
\end{eqnarray*}
\end{proof}

Following Lemma \ref{sec:Lemma3}:

\begin{eqnarray*}
\begin{split}
\sum_{i:n^*_i(T)>n_i(T)}\sum^T_{t=1} Pr(i(t)=i) 1\{n_i(t)\geq n^*_i(t)\}\leq\\ \sum_{i:n^*_i(T)>n_i(T)} \frac{8 ln T}{\Delta_i^{'2}}+ K \frac{\pi^2}{3},
\end{split}
\end{eqnarray*}
where $k=|i|:n^*_i(T)>n_i(T)$.

with

\begin{eqnarray*}
\begin{split}
\sum_{i:n^*_i(T)>n_i(T)}\sum^T_{t=1} Pr(i(t)=i) 1\{n_i(t)\geq n^*_i(t)\}=\\
\sum_{i:n^*_i(T)>n_i(T)} |n^*_i(T)-E[n_i(T)]|
\end{split}
\end{eqnarray*}

\begin{eqnarray*}
\Rightarrow E[R(T)]\leq max_i \mu_i D_{max} \sum_{i:n^*_i(T)>n_i(T)} \frac{8 ln T}{\Delta_i^{'2}}+ K \frac{\pi^2}{3}
\end{eqnarray*}

\section{EXPERIMENTATION}
\label{sec:experimental}

In order to illustrate the strengths and weaknesses of A-UCB in comparison to the state-of-the-art, three synthetic known trend bandit problems are formulated with three reward functions, decreasing reward function (Figure \ref{fig:drec}), sigmoid reward function (Figure \ref{fig:incr}) and Gaussian reward function (Figure \ref{fig:gaus}). 

In these three problems, we have generated a non-stationary arms based on a known shape of the reward function using 8 arms and we have fixed their mean reward as follows: $ \mu_1 = 0.6$, $ \mu_2 = 0.4$, $ \mu_3 = 0.3$, $ \mu_4 = 0.3$, $ \mu_5 = 0.15$, $ \mu_6 = 0.1$, $ \mu_7 = 0.05$, $ \mu_8 = 0.05$, where $ n\in [0,4000]$, $D(n)= -6.65 ln(n)+9.57$ for Figure \ref{fig:drec}, for Figure \ref{fig:incr} and $D(n)= 0.037 exp(1.15 n)$ for Figure \ref{fig:gaus} $D(n)=exp(-(n-20)^2/40)$. Notice that these three simulations with the number of arms and the functions shapes are reflecting our corporate problems.

In this simulation, at each round, if the algorithm chooses the right arm the reward is 1 or else 0. The accumulated rewards are computed each 1000 iteration. The plot of the curves (Figure \ref{fig:drec}, \ref{fig:incr},  \ref{fig:gaus}) are produced by averaging 20 runs of each algorithm. We run the simulation for 32000 iterations.

\begin{figure*}
\begin{multicols}{2}
    \includegraphics[width=\linewidth]{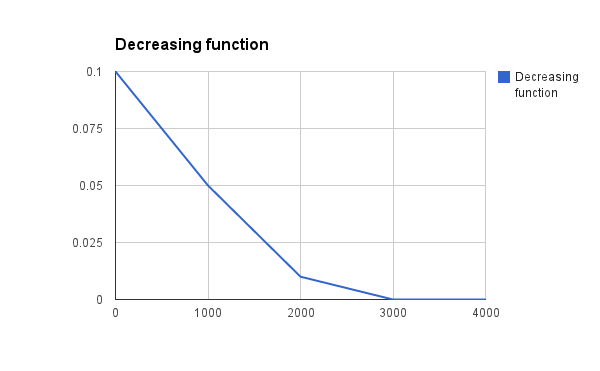}\par 
    \includegraphics[width=\linewidth]{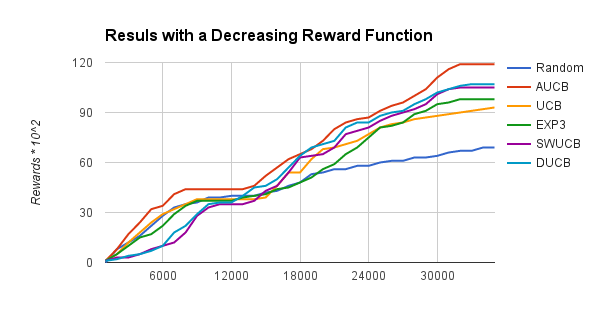}\par 
    \end{multicols}
\caption{Decreasing reward function}\label{fig:drec}
\end{figure*}

\begin{figure*}
\begin{multicols}{2}
    \includegraphics[width=\linewidth]{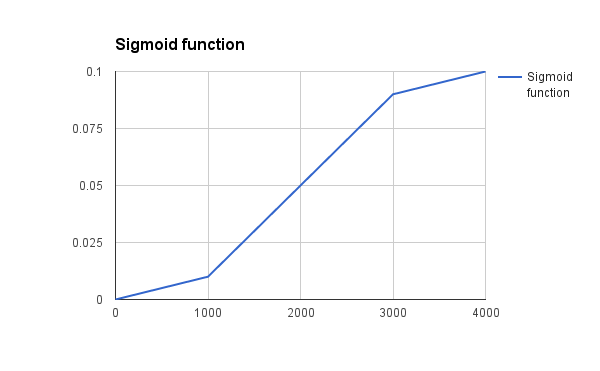}\par 
    \includegraphics[width=\linewidth]{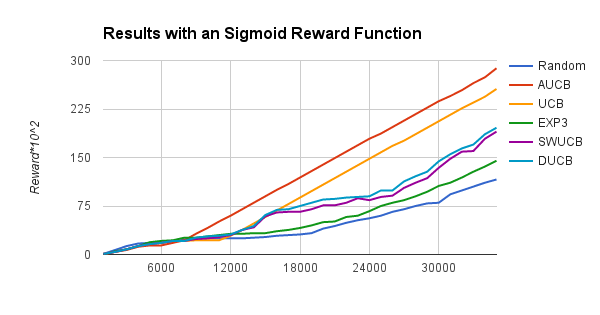}\par 
    \end{multicols}
\caption{Sigmoid reward function}\label{fig:incr}
\end{figure*}

\begin{figure*}
\begin{multicols}{2}
    \includegraphics[width=\linewidth]{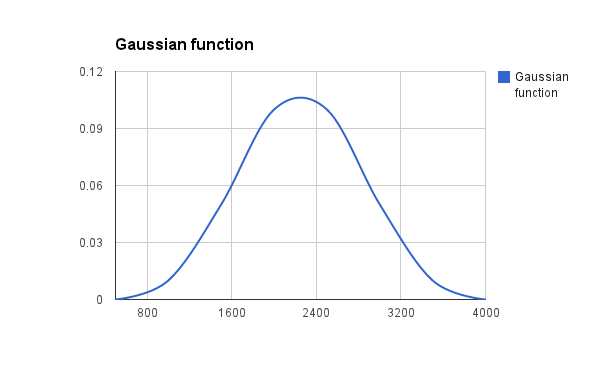}\par 
    \includegraphics[width=\linewidth]{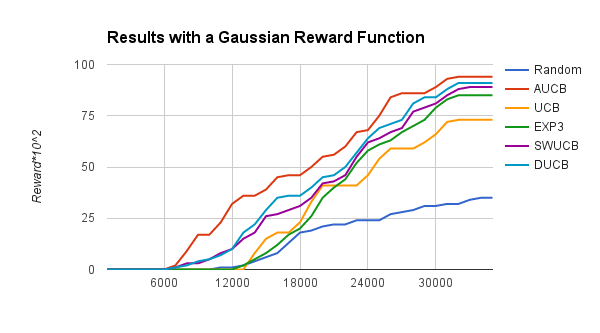}\par 
    \end{multicols}
\caption{Gaussian reward function}\label{fig:gaus}
\end{figure*}

The first problem (Figure \ref{fig:drec}), shows a decreasing reward function, this kind of model can be seen in different real world problem of recommendation system like music or ads recommendation. On this problem, A-UCB quickly outperforms the other algorithms, which is due to the fact that at each iteration the algorithm is aware about the reward function of each arm, this allows it to find the optimal arms at the optimal time. We also observe that an adversarial strategy like EXP3 and switching strategy bandit like SW-UCB are better than a stationary one (UCB), which confirms that the bandits based on a stationary assumption can not solve a non-stationary problem: they perform near the same as a random choice of actions. Another observation is that the higher is the slope of $D(n)$, the smaller is the gap between the A-UCB and the switching bandit algorithm like D-UCB and SW-UCB. Which means that the higher is the slope of $D(n)$ the lower is the convergence of the A-UCB.

In the second problem (Figure \ref{fig:incr}), sigmoid reward function, this kind of model can be seen in different real world problem of recommendation system like interface recommendation. On this problem, A-UCB still outperforms other algorithms and this performance is explained by the rapidity of the A-UCB to find the greatest trade-off between arms. We observe also that UCB outperforms EXP3, SW-UCB and D-UCB after 11000 iterations where the rewards become more stationers. 

In the third problem (Figure \ref{fig:gaus}), the Gaussian reward function can model the reward function of  games or clothes recommendation. In this model we observe that at the increasing part of the rewards function from 0 to 20000 UCB outperform EXP3 but after the EXP3 take over, and that explained by the difficulty of UCB to change his confidence in an arm.

As we can also see in Figure (\ref{fig:drec}, \ref{fig:incr} and \ref{fig:gaus}), that D-UCB performs almost as well as SW-UCB. Both of them waste less time than EXP3 and UCB to detect the breakpoints, and concentrate their pulls on the optimal action.

\section{CONCLUSION}
\label{sec:Conclusion}

We have introduced a new formulation of the MAB problem motivated by the real world problem of active learning, music and interface recommendation. In this setting the set of strategies available to a MAB algorithm changes rapidly over time. We provided an extension that allows UCB algorithm to be used in the case of MAB problem with known trend. Further, we provided an upper bound of regret of the proposed algorithm. Simulations on multiple payoff distributions shows the performance of the proposed algorithm.

%unsrt
%ieeetr
\bibliographystyle{ieeetr}
\bibliography{biblio}

\begin{thebibliography}{10}

\bibitem{BouneffoufLUFA14}
D.~Bouneffouf, R.~Laroche, T.~Urvoy, R.~Feraud, and R.~Allesiardo, ``Contextual
  bandit for active learning: Active thompson sampling,'' in {\em Neural
  Information Processing - 21st International Conference, {ICONIP} 2014,
  Kuching, Malaysia, November 3-6, 2014. Proceedings, Part {I}}, pp.~405--412,
  2014.

\bibitem{hu2011nextone}
Y.~Hu and M.~Ogihara, ``Nextone player: A music recommendation system based on
  user behavior.,'' in {\em ISMIR}, pp.~103--108, 2011.

\bibitem{rosman2014user}
B.~Rosman, S.~Ramamoorthy, M.~H. Mahmud, and P.~Kohli, ``On user behaviour
  adaptation under interface change.,'' 2014.

\bibitem{LaiRobbins1985}
T.~L. Lai and H.~Robbins, ``Asymptotically efficient adaptive allocation
  rules,'' {\em Advances in Applied Mathematics}, vol.~6, no.~1, pp.~4--22,
  1985.

\bibitem{4}
P.~Auer, N.~Cesa-Bianchi, and P.~Fischer, ``Finite-time analysis of the
  multiarmed bandit problem,'' {\em Mach. Learn.}, vol.~47, no.~2-3,
  pp.~235--256, 2002.

\bibitem{AllesiardoFB14}
R.~Allesiardo, R.~Feraud, and D.~Bouneffouf, ``A neural networks committee for
  the contextual bandit problem,'' in {\em Neural Information Processing - 21st
  International Conference, {ICONIP} 2014, Kuching, Malaysia, November 3-6,
  2014. Proceedings, Part {I}}, pp.~374--381, 2014.

\bibitem{garivier2008upper}
A.~Garivier and E.~Moulines, ``On upper-confidence bound policies for
  non-stationary bandit problems,'' {\em arXiv preprint arXiv:0805.3415}, 2008.

\bibitem{MellorS13}
J.~Mellor and J.~Shapiro, ``Thompson sampling in switching environments with
  bayesian online change detection,'' in {\em Proceedings of the Sixteenth
  International Conference on Artificial Intelligence and Statistics, {AISTATS}
  2013, Scottsdale, AZ, USA, April 29 - May 1, 2013}, pp.~442--450, 2013.

\bibitem{gittins1979bandit}
J.~C. Gittins, ``Bandit processes and dynamic allocation indices,'' {\em
  Journal of the Royal Statistical Society. Series B (Methodological)},
  pp.~148--177, 1979.

\bibitem{whittle1988restless}
P.~Whittle, ``Restless bandits: Activity allocation in a changing world,'' {\em
  Journal of applied probability}, pp.~287--298, 1988.

\bibitem{nino2001restless}
J.~Nino-Mora {\em et~al.}, ``Restless bandits, partial conservation laws and
  indexability,'' {\em Advances in Applied Probability}, vol.~33, no.~1,
  pp.~76--98, 2001.

\bibitem{Hartland2006}
S.~G. O. T. M.~S. C.~Hartland, N.~Baskiotis, ``Multi-armed bandit, dynamic
  environments and meta-bandits,'' {\em Online Trading of Exploration and
  Exploitation Workshop}, vol.~1, no.~10, pp.~1--34, 2006.

\bibitem{BurtiniLL15}
G.~Burtini, J.~Loeppky, and R.~Lawrence, ``Improving online marketing
  experiments with drifting multi-armed bandits,'' in {\em {ICEIS} 2015 -
  Proceedings of the 17th International Conference on Enterprise Information
  Systems, Volume 1, Barcelona, Spain, 27-30 April, 2015}, pp.~630--636, 2015.

\bibitem{NIPSMortal1}
D.~Chakrabarti, R.~Kumar, F.~Radlinski, and E.~Upfal, ``Mortal multi-armed
  bandits,'' pp.~273--280, 2009.

\bibitem{kanade2009sleeping}
V.~Kanade, H.~B. McMahan, and B.~Bryan, ``Sleeping experts and bandits with
  stochastic action availability and adversarial rewards,'' in {\em
  International Conference on Artificial Intelligence and Statistics},
  pp.~272--279, 2009.

\end{thebibliography}

\end{document}